\documentclass[twoside]{article}

\usepackage[accepted]{aistats2023}
\usepackage[round]{natbib}

\usepackage{multibib}
\newcites{latex}{References}

\usepackage{amsthm}
\newtheorem{theorem}{Theorem}
\newtheorem{corollary}{Corollary}
\newtheorem{lemma}{Lemma}
\newtheorem{definition}{Definition}
\newtheorem{remark}{Remark}
\usepackage{amssymb}

\usepackage{xcolor}     
\usepackage{algorithm}
\usepackage{algpseudocode}
\usepackage{graphicx}
\usepackage[hidelinks]{hyperref}
\usepackage{float}

\usepackage[capitalize,noabbrev]{cleveref}

\begin{document}

%

%

\twocolumn[

\aistatstitle{Randomized Greedy Learning for Non-monotone Stochastic Submodular Maximization Under Full-bandit Feedback}

\aistatsauthor{ Fares Fourati \And Vaneet Aggarwal \And  Christopher John Quinn \And Mohamed-Slim Alouini}

\aistatsaddress{ KAUST \\ fares.fourat@kaust.edu.sa \And Purdue University \& KAUST \\ vaneet.aggarwal@kaust.edu.sa \And  Iowa State University \\ cjquinn@iastate.edu \And KAUST \\ slim.alouini@kaust.edu.sa} 
]

\begin{abstract}
We investigate the problem of unconstrained combinatorial multi-armed bandits with full-bandit feedback and stochastic rewards for submodular maximization. Previous works investigate the same problem assuming a submodular and monotone reward function. In this work, we study a more general problem, i.e., when the reward function is not necessarily monotone, and the submodularity is assumed only in expectation. We propose Randomized Greedy Learning (RGL) algorithm and theoretically prove that it achieves a $\frac{1}{2}$-regret upper bound of $\Tilde{\mathcal{O}}(n T^{\frac{2}{3}})$ for horizon $T$ and number of arms $n$. We also show in experiments that RGL empirically outperforms other full-bandit variants in submodular and non-submodular settings.

\end{abstract}

\section{Introduction}
\label{into}

The stochastic multi-armed bandits, first introduced by \cite{robbins1952some}, formalizes several challenging decision-making problems, such as clinical decisions, investment, pricing, influence maximization, and product recommendation. The goal of a decision-maker 
can be modeled as maximizing a particular reward function that depends on her decisions throughout time. The decision maker needs to tradeoff between exploration (exploring sub-optimal arms) and exploitation (playing the chosen arm), and efficient guarantees for regret have been widely studied \citep{thompson1933likelihood, auer2002using, auer2002finite,  auer2010ucb, agrawal2012analysis, pmlr-v32-gopalan14}. 

One natural extension for the multi-armed bandit problem is the combinatorial multi-armed bandit problem. At each round, instead of selecting just one base arm, the agent selects a set of arms and receives a joint reward for that set. If the agent only receives the reward for a chosen set of arms, then it is called \textit{full-bandit} feedback. Otherwise, if the agent receives further information about his choice, such as the reward of each individual arm of that set, then it is called \textit{semi-bandit feedback}. The former setting is more challenging, as the decision maker has far less information to decide than in the latter. The former setting is the focus of this paper. 

The study of combinatorial multi-armed bandits problems with submodular reward functions has recently attracted much attention \citep{nie2022explore, 49310}. Formally, a set function $f: 2^{\Omega} \rightarrow \mathbb{R}$ defined on a finite ground set $\Omega$ is said to be submodular if it satisfies the diminishing return property: for all $A \subseteq B \subseteq \Omega$, and $x \in \Omega \backslash B$, it holds that $f(A \cup\{x\})-f(A) \geq f(B \cup\{x\})-f(B)$. The submodularity assumption is motivated by several real-world scenarios. For example, opening more supermarkets in a certain area would result in diminishing returns due to demand saturation. Hence, the widespread use of submodular functions as utility functions in economics and algorithmic game theory. Furthermore, submodularity appears in many important settings in combinatorial optimization such as cuts in graphs \citep{goemans1995improved, iwata2001combinatorial}, rank functions of matroids \citep{edmonds2003submodular}, and set covering problems \citep{feige1998threshold}.

Multi-armed bandits have been studied in two different settings, \textit{adversarial setting} where an adversary generates a reward sequence potentially based on the agent's previous decisions \citep{auer2002nonstochastic}, and \textit{stochastic setting} where the reward of each action is drawn independently from a certain (unknown) distribution \citep{auer2002finite}. An adversarial setting is harder for standard multi-armed bandits, and its result can be directly used as one achievable strategy for the stochastic setting \citep{lattimore2020bandit}. However, the same is not true for the research on submodular bandits. In prior works in the area \citep{pmlr-v75-roughgarden18a, 49310}, the environment in adversarial bandits chooses a sequence of submodular functions $\left\{f_1, . . . , f_T \right\}$. In this work, we focus on stochastic reward functions. Thus, we assume a more relaxed property of submodularity which is submodularity in expectation (as defined in Definition \ref{def:submod}). That is, the realizations of the stochastic function $f_t$ in the problem we consider need not be submodular, making the adversarial algorithms no longer hold in this setting. 

\begin{definition}\label{def:submod}
 A stochastic set function $f: 2^{\Omega} \rightarrow \mathbb{R}$ defined on a finite ground set $\Omega$ is said to be submodular in expectation if it satisfies the diminishing return property in expectation: for all $A \subseteq B \subseteq \Omega$, and $x \in \Omega \backslash B$, we have,
 \begin{equation}
     \mathbb{E}f(A \cup\{x\})-\mathbb{E}f(A) \geq \mathbb{E}f(B \cup\{x\})-\mathbb{E}f(B).
 \end{equation}
\end{definition}

Several works in the literature assume submodular monotone functions, as it is simpler to manipulate and can have stronger guarantees \citep{nie2022explore,chen2020black}. A submodular set function $f: 2^{\Omega} \rightarrow \mathbb{R}$ is called monotone if for any $A \subseteq B \subseteq \Omega$ we have $f(A) \leq f(B)$. This work considers a more general problem where the functions are not necessarily monotone. 

There are several motivating use cases for the non-monotone submodular maximization, including optimizing feature selection \citep{das2008algorithms, khanna2017scalable, elenberg2018restricted}, and data summarization \citep{mirzasoleiman2016fast}. Optimizing feature selection can be modeled as a non-monotone submodular maximization due to the possible overfitting to the training data \citep{fahrbach2018non}. Data summarization selects a representative subset of data points, and the typical utility functions are submodular while not monotone to penalize larger solutions \citep{tschiatschek2014learning, dasgupta-etal-2013-summarization}. For further motivating examples, see Appendix \ref{appendix_motivation}.

{\bf Contributions:} The key contributions in this paper are summarized as follows:

\noindent {\bf i.} We propose Randomized Greedy Learning (RGL), the first algorithm designed for stochastic combinatorial multi-armed bandits problems with a non-monotone stochastic submodular reward function and full-bandit feedback. It has low storage and computational complexity.

\noindent {\bf ii.} We prove that RGL achieves a $\frac{1}{2}$-regret upper bound gurantees of $\mathcal{O}(n T^{\frac{2}{3}} \log(T)^{\frac{1}{3}})$ for horizon $T$ and number of arms $n$.

\noindent {\bf iii.} We empirically show that RGL outperforms other full-bandit feedback variants regarding expected reward and cumulative regret.

{\bf Related Work:} Submodular maximization is NP-hard.  \cite{feige2011maximizing} showed that for any constant $\varepsilon > 0$, any algorithm achieving an approximation of $(\frac{1}{2} + \varepsilon)$ requires an exponential number of oracle queries to the non-monotone submodular function. Further, they proposed several greedy algorithms, such as deterministic adaptive and randomized adaptive, that are $\frac{1}{3}$ and $\frac{2}{5}$ -approximation algorithms, respectively. More recently \citep{buchbinder2015tight, buchbinder2018deterministic} proposed linear time $\frac{1}{2}$-approximation algorithms. Our work extends the greedy algorithms in \citep{buchbinder2015tight} from the non-stochastic offline setting to the stochastic online setting and proves the regret guarantees in the stochastic online setup. In practical scenarios, rewards are stochastic; thus, the agent has to optimize online exploration and exploitation under noisy rewards. The online setting requires an exploration-exploitation tradeoff for efficient regret guarantees, with samples from the stochastic function, making the problem more complex. 

Non-monotone submodular maximization has recently been studied  in the adversarial setting \citep{pmlr-v75-roughgarden18a}, where a greedy algorithm  under full-information is proposed which achieves a $\frac{1}{2}$-regret upper bound of $\Tilde{\mathcal{O}}(n T^{\frac{1}{2}})$. Apart from the differences in the stochastic and adversarial settings, our work is also different from a feedback perspective. While they study the problem under full-information, namely after playing an action $S_t$, they receive not only the reward $f_t(S_t)$ but the entire function $f_t(\cdot)$.  We study the problem under full-bandit feedback, i.e., the agent, in our case, has much less information to make decisions. Full-bandit feedback in the adversarial setting has also been recently studied \citep{49310} where the proposed algorithm achieves a $\frac{1}{2}$-regret upper bound of $\Tilde{\mathcal{O}}(n T^{\frac{2}{3}})$.

\vspace{-.1in}
\section{Problem Statement}
\label{problem}
\vspace{-.1in}

In this section, we formally define the problem studied in this paper. Let $\Omega$ be the set of all the base arms, $u_i$ an arm of index $i$, and $n = |\Omega|$ be the number of arms. We consider a sequential decision-making problem with a fixed horizon $T$, where at each time step $t$, the agent chooses a subset of arms (action), $S_t \subseteq \Omega$. At every step $t$, the agent receives a sample reward for selecting a subset using a stochastic function $f(S_t)$. 

We assume the reward function $f(\cdot)$ to be stochastic and submodular in expectation,  see Definition \ref{def:submod}, not necessarily monotone, and i.i.d. conditioned on a given subset. Without loss of generality, we assume $f(\cdot)$ to be bounded in $[0,1]$ \footnote{The results can be directly extended to  a general submodular in expectation function $f(\cdot)$ with a minimum value $f_{min}$ and a maximum values $f_{max}$ by considering a normalized submodular function $g(S) = (f(S) - f_{min})/(f_{max}- f_{min})$}. 
The agent's goal is to maximize the cumulative reward over time until the time horizon $T$.

One standard metric to measure the performance of an online learner over time is to compare its performance with an agent that has access to the optimal maximizer $OPT$ of the expectation of the reward function $f(\cdot)$,  

\vspace{-.2in}
\begin{equation}
    OPT = \text{arg} \max_{S \subseteq \Omega} \mathbb{E} f(S).
\end{equation}
Maximizing a general non-monotone submodular function is an NP-hard problem. \cite{feige2011maximizing} studied the hardness of non-monotone submodular maximization assuming the function $f(\cdot)$ is obtained through a value oracle. They proved that for any constant $\varepsilon > 0$, any algorithm achieving an approximation of $(\frac{1}{2} + \varepsilon)$ requires an exponential number of oracle queries. Subsequently, \cite{buchbinder2015tight} achieved the $\frac{1}{2}$-approximation in linear time in the offline non-stochastic setting. Therefore, we compare the agent's cumulative reward to $\frac{1}{2}T \mathbb{E} f(OPT)$, and we denote the cumulative $\frac{1}{2}$-regret $\mathcal{R}_{\frac{1}{2}}(T)$, where, 
\vspace{-.1in}
\begin{equation}
\begin{aligned}
\mathcal{R}_{\frac{1}{2}} &= \sum_{t=1}^{T} (\frac{1}{2} f(OPT) - f(S_t))
\end{aligned}
\end{equation}
Notice that the $\frac{1}{2}$-regret is random, and its randomness is due to the stochasticity of the reward function $f(\cdot)$ and the chosen actions (subsets) throughout time. Thus, we mainly focus on minimizing the expected $\frac{1}{2}$-regret of the agent, defined as follows, 
\begin{equation}
\begin{aligned}
\label{exp_regret_definition}
\mathbb{E}[\mathcal{R}_{\frac{1}{2}}] &= \frac{1}{2}T \mathbb{E} [f(OPT)] - \sum_{t=1}^{T} \mathbb{E}[f(S_t)],
\end{aligned}
\end{equation}
where the expectation is defined over the stochasticity of $f(\cdot)$ and the randomness of the chosen sequence of actions, for ease of notation, we write $\mathcal{R}(T)$ instead of $\mathcal{R}_{\frac{1}{2}}(T)$ for the remainder of the paper.

\section{Proposed RGL Algorithm}
\label{algo}

This section presents our proposed RGL algorithm, adapted from the offline algorithm proposed in \citep{buchbinder2015tight} for a non-stochastic $f(\cdot)$. The pseudocode for RGL can be found in Algorithm \ref{alg:RGL}. 

For the problem we consider (unconstrained action space), if the reward function is monotone, then the best set is simply the set of all base arms $\Omega$. A trivial algorithm (no exploration needed) can attain an approximation ratio of $\alpha=1$. However, when the reward function is non-monotone, adding arms is no longer necessarily a good choice. Consequently, tracking two sets $X_i$ (starting as $\emptyset$) and $Y_i$ (starting as $\Omega$) is a useful strategy. RGL goes over all the individual arms one by one and decides whether to add it to a set of base arms $X_i$ or remove it from the set of base arms $Y_i$. The decisions of adding or removing any arm are made in a randomized greedy fashion using empirical estimates of marginal gains until a decision is made for all the individual arms and then exploits the decided best set of arms. 

Let $X_i$ and $Y_i$ be two sets of arms. Initially, $X_0 = \emptyset$ and $Y_0 = \Omega$. The algorithm has $n$ phases, where $n$ is the number of arms, and each phase has $m$ sub-phases, where $m$ is the number of repetitions to estimate the quality of a given set of arms. In phase $i$ out of $n$, the agent estimates the expectation of the following two random variables, $a_i$ and $b_i$, defined as follows,
\begin{equation}
\begin{aligned}
    &a_i = f(X_{i-1} \cup \left\{u_i\right\}) - f(X_{i-1}) \\
    &b_i =  f(Y_{i-1} \setminus \left\{u_i\right\}) - f(Y_{i-1}).
\end{aligned}\label{eq:abdefn}
\end{equation}

\begin{algorithm}
\caption{RGL}\label{alg:RGL}
\begin{algorithmic}
\Require Set of base arms $\Omega$, horizon $T$  \\
$X_{0} \leftarrow \emptyset, Y_{0} \leftarrow \Omega, n \leftarrow|\Omega|$ \\
$m \leftarrow\lceil\left(T \sqrt{\frac{25}{32} \log (T)}\right)^{2 / 3}\rceil$
\For{arm index $i \in\{1, \cdots, n\}$}
    \State $\bar{a}_i \leftarrow 0$ and $\bar{b}_i \leftarrow 0$
    \For{sample $j \in\{1, \ldots, m\}$}
        \State Play $X_{i-1} \cup \left\{u_i\right\}$, $X_{i-1}$, $Y_{i-1}$, and $Y_{i-1} \setminus \left\{u_i\right\}$
        \State $\bar{a}_i \leftarrow \bar{a}_i + (f_j(X_{i-1} \cup \left\{u_i\right\}) - f_j(X_{i-1}))/m$
        \State $\bar{b}_i \leftarrow \bar{b}_i + (f_j(Y_{i-1} \setminus \left\{u_i\right\}) - f_j(Y_{i-1}))/m$
    \EndFor 
    \State $a_i^{\prime} \leftarrow \max(\bar{a}_i, 0)$ and $b_i^{\prime} \leftarrow \max(\bar{b}_i, 0)$
    \State \textbf{with probability} $(\frac{a_i^{\prime}}{a_i^{\prime} + b_i^{\prime}})$ \textbf{do} 
    \State $\quad$ $X_i \leftarrow X_{i-1} \cup \left\{u_i\right\}$ and $Y_i \leftarrow Y_{i-1}$
    \State \textbf{else}
    \State $\quad$ $Y_i \leftarrow Y_{i-1} \setminus \left\{u_i\right\}$ and $X_i \leftarrow X_{i-1}$
\EndFor

\For{remaining time} 
    \State Play $X_n$ 
\EndFor 
\end{algorithmic}
\end{algorithm}

Since $f(\cdot)$ is stochastic, $a_i$ and $b_i$ are too, even when conditioned on $X_{i-1}$, $Y_{i-1}$, and the arm $u_i$. To estimate their expectations given the sets $X_i$ and $Y_i$ and the arm $u_i$, the agent samples each of the four random set values in \eqref{eq:abdefn} $m$ times. 
Denote the $j$th sample of a played set $S$ as $f_j(S)$ and the empirical mean of playing an action $S$ as follows, 
\begin{equation}
\label{empiricalmean}
    \bar{f}(S) :=  \frac{1}{m} \sum_{j = 1}^{m} f_j(S).
\end{equation}  
Hence, the agent computes their empirical means $\bar{a}_i$ and $\bar{b}_i$ over $m$ repetitions, i.e., 
\begin{equation}
\begin{aligned}
    &\bar{a}_i =\bar{f}(X_{i-1} \cup \left\{u_i\right\}) - \bar{f}(X_{i-1}) \\
    &\bar{b}_i =  \bar{f}(Y_{i-1} \setminus \left\{u_i\right\}) - \bar{f}(Y_{i-1}).
\end{aligned}\label{eq:abdefnemp}
\end{equation}
These two estimates are important for the decision-making process. $\bar{a}_i$ measures the expected impact of adding arm $u_i$ to $X_{i-1}$, while $\bar{b}_i$ measures the expected impact of removing arm $u_i$ from $Y_{i-1}$. A decision is made greedily and probabilistically by computing a certain probability that depends on these two estimates $\bar{a}_i$ and $\bar{b}_i$, defined as follows, 
\begin{equation}
    p = \frac{a_i^{\prime}}{a_i^{\prime} + b_i^{\prime}},
\end{equation}
where $a_i^{\prime} = \max(\bar{a}_i, 0)$ and $b_i^{\prime} = \max(\bar{b}_i, 0)$, which explains the randomized greedy name of the algorithm. In the special case when $a_i^{\prime} =  b_i^{\prime} =  0$, we set $p = 1$. 

With that probability $p$, the agent adds the individual arm $i$ to the set of arms $X_i$ and keeps it in the set of arms $Y_i$, and with probability $1 - p$, the agent removes the arm $i$ from the set of arms $Y_i$ and keeps the same arms as in $X_{i-1}$. Thus, $X_i \subseteq Y_i$ for all $i=1,\dots,n$.  After checking all the $n$ individual arms, it can be easily seen that by the algorithm's construction, both sets $X_n$ and $Y_n$ contain exactly the same arms, i.e., $X_n = Y_n$. Thus, after deciding on each of the $n$ base arms, the agent exploits $X_n$ for the remaining time.

\begin{remark}
Note that the randomness of our randomized greedy algorithm was essential to achieve the $1/2$ approximation guarantees. The same algorithm with deterministic decisions, i.e., adding arm of index $i$ when $a_i \geq b_i$ would only achieve $1/3$ approximation guarantee, \citep{buchbinder2015tight}.
\end{remark}

RGL has low storage complexity and per-round time complexity. During exploitation, RGL only needs to store the indices of the selected set $X_n$ of base arms, which is at most $n$ and does not need further computation. During exploration, in phase $i$, RGL needs to update the empirical means for $\bar{a}_i$ and $\bar{b}_i$, and update the $X_i$ and $Y_i$. Thus, RGL has an $\mathcal{O}(n)$ storage complexity and  $\mathcal{O}(1)$ per round time complexity.

\section{Regret Analysis}

In this section, we will provide the paper's main result, which is a bound on the expected cumulative $\frac{1}{2}$-regret of the proposed algorithm. Before we mention the main result, we provide the Lemmas that will be useful in proving the main result. 

\begin{lemma}
For every $i \in \{1, \cdots,  n\}$, we have $\mathbb{E}[a_i + b_i] \geq 0$, where $a_i$, $b_i$ are as defined in \eqref{eq:abdefn}.  
\label{lem:abnonneg}
\end{lemma}
\begin{proof} 
By construction, $X_{i-1} \subseteq Y_{i-1}\backslash\{u_i\}$ and $u_i\in Y_{i-1}$.   Thus, by \cref{def:submod} of submodularity, the expected marginal gain of adding $u_i$ to $Y_{i-1}\backslash\{u_i\}$ is less than or equal to the marginal gain of adding $u_i$ to $X_{i-1}$, 
\begin{align}
    &\hspace{-.5cm}\mathbb{E}\left[f\left(Y_{i-1}\right) - f\left(Y_{i-1} \backslash\left\{u_i\right\}\right)\right] \nonumber\\
    &\leq \mathbb{E}[f\left(X_{i-1} \cup\left\{u_i\right\}\right)-f\left(X_{i-1}\right)]. \label{eq:lem1:submod}
\end{align}
Plugging \eqref{eq:abdefn}   into \eqref{eq:lem1:submod} yields $\mathbb{E}[-b_i]\leq \mathbb{E}[a_i]$ which upon rearranging finishes the proof.

\end{proof}

For each arm $u_i$, the agent plays the following list of actions $\mathcal{S}_i = \left[X_{i-1}, X_{i-1} \cup\left\{u_i\right\}, Y_{i-1}, Y_{i-1} \backslash\left\{u_i\right\}\right]$ 
exactly $m$ times, then computes marginal gain estimates. To determine $m$, we  consider the equal-sized confidence radii $\mathrm{rad}:=\sqrt{2 \log (T) / m}$ for empirical estimates for all the actions $S_i$. Increasing $m$ improves the concentration of empirical estimates around their mean values, improving the quality of decisions made using those empirical estimates. However, increasing $m$ comes at the cost of more time spent playing actions whose values may be far from $\frac{1}{2}f(OPT)$ leading to high cumulative regret.

Denote the event that the empirical means of actions played when testing arm $u_i$ are concentrated around their statistical means as,
\begin{equation}
\label{concentration}
    \mathcal{E}_i:=\bigcap_{S \in \mathcal{S}_i}\{|\bar{f}(S)-\mathbb{E}[\bar{f}(S)]|<\operatorname{rad}\}
\end{equation}
Then we define the clean event $\mathcal{E}$ to be the event that the empirical means of all actions played up to and including arm $u_n$ are within rad of their corresponding statistical means:
\begin{equation}
\label{Eq:bigE}
    \mathcal{E}:=\mathcal{E}_1 \cap \cdots \cap \mathcal{E}_n .
\end{equation}  The specific sequence of actions played will depend on empirical estimates of earlier actions and their rewards. However, conditioned on the current action $S_t$ played at any time $t$, the random reward $f(S_t)$ is independent of past actions and their rewards.

Using the Hoeffding bound, we show that $\mathcal{E}$ happens with high probability. We then use the concentration of empirical means (\ref{concentration}) and properties of submodularity in expectation, Definition \ref{def:submod}, to show the next steps.

\begin{remark}
Under the clean event $\mathcal{E}_i$ (\ref{concentration}), for all $S \in \mathcal{S}_i$,
$$
|\bar{f}(S)-\mathbb{E}[\bar{f}(S)]|<\operatorname{rad}.
$$
Thus, since $X_{i-1}$ 
is in $\mathcal{S}_i$, 
$$
\mathbb{E}[\bar{f}(X_{i-1})] -rad \leq \bar{f}(X_{i-1}) \leq \mathbb{E}[\bar{f}(X_{i-1})] + \operatorname{rad}.
$$
We have similar relation for $X_{i-1} \cup\left\{u_i\right\}$, $Y_{i-1}$, $Y_{i-1}\backslash\{u_i\}$. 

Thus, 
\begin{align}
&\hspace{-1cm} \mathbb{E}[\bar{f}(X_{i-1} \cup\left\{u_i\right\})] - \nonumber \mathbb{E}[\bar{f}(X_{i-1})] -2 rad \\ \nonumber
&\leq \bar{f}(X_{i-1} \cup\left\{u_i\right\}) - \bar{f}(X_{i-1}) \\ \nonumber
&= \frac{1}{m} \sum_{j = 1}^{m} (f_j(X_{i-1} \cup\left\{u_i\right\}) - f_j(X_{i-1})) \\ \nonumber
&= \bar{a}_i. \tag{by \eqref{eq:abdefnemp}}
\end{align}
Therefore, 
\begin{equation}
    \mathbb{E}[a_i] - 2 rad \leq  \bar{a}_i \nonumber
\end{equation}
Using similar steps, it can be easily verified that, 
\begin{equation}
\label{ai_bi_bounds}
\begin{aligned}
    &\mathbb{E}[a_i] - 2 rad \leq  \bar{a}_i \leq \mathbb{E}[a_i] + 2 rad \\
    &\mathbb{E}[b_i] - 2 rad \leq  \bar{b}_i \leq \mathbb{E}[b_i] + 2 rad.
\end{aligned}
\end{equation}

\end{remark}
\color{black}


\begin{corollary}
\label{corollary4rad}
Under the clean event $\mathcal{E}$, for every $1 \leq i \leq n$, $\bar{a}_i + \bar{b}_i \geq -4 rad$.
\end{corollary}
\begin{proof}
Under clean event $\mathcal{E}$, $\bar{a}_i \geq \mathbb{E}[a_i] - 2rad$ and $\bar{b}_i \geq \mathbb{E}[b_i] - 2rad$. Since $\mathbb{E}[a_i + b_i] \geq 0$ (by Lemma \ref{lem:abnonneg}), then, \\
$\bar{a}_i + \bar{b}_i \geq \mathbb{E}[a_i + b_i] - 4rad \geq - 4rad$.
\end{proof}

\begin{lemma}
\label{lemma 5rad}
Define $OPT_{i} := \left(OPT \cup X_i\right) \cap Y_i$. Under the clean event $\mathcal{E}$, for every $1 \leq i \leq n$, we have
\begin{equation}\label{ineq:lemma 5rad}
    \begin{aligned}
    &\mathbb{E}[f(OPT_{i-1}) - f(OPT_i))]\\
    &\leq \frac{1}{2} \mathbb{E}[f(X_i) - f(X_{i-1}) + f(Y_i) - f(Y_{i-1})]+ 5 rad .
    \end{aligned}
\end{equation}
\end{lemma}
\begin{proof}
It is sufficient to prove the inequality conditioned on any event of the form $X_{i-1} = S_{i-1}$ where $S_{i-1} \subseteq \left\{u_1, ..., u_{i-1} \right\}$, for which the probability $X_{i-1} = S_{i-1}$ is non-zero.  
The remainder of the proof assumes everything is conditioned on this event. 
We prove \cref{lemma 5rad} by considering the following four possible cases for $\bar{a}_{i}$ and $\bar{b}_{i}$:\\

$\quad$ \textbf{Case 1} ($\bar{a}_i \geq 0$ and $\bar{b}_i\leq 0$): In this case $\bar{b}_i\leq 0 \Rightarrow b_i^{\prime} = 0 \Rightarrow \frac{a_i^{\prime}}{a_i^{\prime}+ b_i^{\prime}} = 1$. Thus,  $Y_i = Y_{i-1}$ and $X_{i} = X_{i-1} \cup \left\{u_i\right\}$.  Since $X_{i} = X_{i-1} \cup \left\{u_i\right\}$, we have 
\begin{equation}
\begin{aligned}
\mathbb{E}[a_i] &= \mathbb{E}[f(X_{i-1} \cup \left\{u_i\right\}) - f(X_{i-1})] \\
&= \mathbb{E}[f(X_{i}) - f(X_{i-1})].
\end{aligned}\label{eq:eai:case1}
\end{equation}
Since $Y_i = Y_{i-1}$, the relation \eqref{ineq:lemma 5rad} that we want to show  reduces to,
\begin{equation*}
    \begin{aligned}
    &\hspace{-1cm}\mathbb{E}[f(OPT_{i-1}) - f(OPT_i)]  \\&\leq \frac{1}{2} \mathbb{E}[f(X_i) - f(X_{i-1})]+ 5 rad .
    \end{aligned}
\end{equation*}
Notice, $OPT_i = \left(OPT \cup X_i\right) \cap Y_i = OPT_{i-1} \cup \left\{u_i\right\}$. 

If $u_i \in OPT \Rightarrow OPT_i = OPT_{i-1}$. Thus, 
\begin{align}
    &\hspace{-1cm}\mathbb{E}[f(OPT_i) - f(OPT_{i-1})] \nonumber\\
    &= 0 \nonumber\\
    &\leq \frac{\bar{a}_i}{2} \tag{by case 1 condition} \\
    &\leq \frac{\mathbb{E}[a_i]}{2} + rad \tag{using (\ref{ai_bi_bounds})}\\
    &
    =\frac{1}{2}\mathbb{E}[f(X_i) - f(X_{i-1})] + rad \tag{by \eqref{eq:eai:case1}}\\
    &\leq \frac{1}{2}\mathbb{E}[f(X_i) - f(X_{i-1})] + 5rad. \nonumber
\end{align}

Now consider that $u_i \notin OPT$.  Since $OPT_{i-1} \subseteq Y_{i-1}$ holds by definition of $OPT_i$, here  $OPT_{i-1} \subseteq Y_{i-1}\backslash\{u_i\}$, and $(Y_{i-1} \backslash\left\{u_i\right\}) \cup\left\{u_i\right\}=Y_{i-1}$, so by \cref{def:submod} of submodularity in expectation, %
\begin{equation}\label{eq:prf:case1:14}
    \begin{aligned}
    &\hspace{-.5cm}\mathbb{E} [ f\left(Y_{i-1}\right)]-\mathbb{E}[ f\left(Y_{i-1} \backslash \left\{u_i\right\}\right) ]\\
    &\leq  \mathbb{E}[f(OPT_{i-1} \cup \left\{u_i\right\})] -\mathbb{E}[f(OPT_{i-1})]. 
    \end{aligned}
\end{equation}
Negating \eqref{eq:prf:case1:14}, we obtain 
\begin{align}
    &\hspace{-0.5cm}\mathbb{E}[f(OPT_{i-1})] - \mathbb{E}[f(OPT_{i-1} \cup \left\{u_i\right\})] \nonumber\\
    &\leq  \mathbb{E}[ f\left(Y_{i-1} \backslash \left\{u_i\right\}\right)]  - \mathbb{E}[ f\left(Y_{i-1}\right)] \tag{negating \eqref{eq:prf:case1:14}} \\
    &= \mathbb{E}[b_i] \tag{by def. of $b_i$ \eqref{eq:abdefn}}\\
    &\leq \bar{b}_i + 2rad \tag{using (\ref{ai_bi_bounds})}\\
    &\leq \frac{\bar{a}_i}{2} + 2rad \tag{condition for case 1}\\
    &\leq \frac{1}{2}\mathbb{E}[a_i] + 3rad \tag{using (\ref{ai_bi_bounds})}\\
    %
    &= \frac{1}{2}\mathbb{E}[f(X_i) - f(X_{i-1})] + 3rad \tag{by \eqref{eq:eai:case1}}\\
    &\leq \frac{1}{2}\mathbb{E}[f(X_i) - f(X_{i-1})] + 5rad. \nonumber
\end{align}

$\quad$\textbf{Case 2} ($\bar{a}_i < 0$ and $\bar{b}_i\geq 0$): This case is analogous to Case 1, for its proof, we refer the reader to Appendix  \ref{lemma_2_case_2}. 

$\quad$ \textbf{Case 3} ($\bar{a}_i < 0$ and $\bar{b}_i < 0$): For this case, by definition of $a_i'$ and $b_i'$, we will have $a_i^{\prime} =  b_i^{\prime} =  0$.  Thus, the selection probability of arm $u_i$ will be set as $\frac{a_i^{\prime}}{a_i^{\prime} + b_i^{\prime}} = 1$, meaning $X_{i} = X_{i-1} \cup \left\{u_i\right\}$ and $Y_i = Y_{i-1}$. Hence, we have  
\begin{equation}
\begin{aligned}
\mathbb{E}[a_i] &= \mathbb{E}[f(X_{i-1} \cup \left\{u_i\right\}) - f(X_{i-1})] \\
&= \mathbb{E}[f(X_{i}) - f(X_{i-1})].
\end{aligned}\label{eq:eai}
\end{equation}
Thus, it suffices to prove that
\begin{equation}
    \mathbb{E}(f(OPT_{i-1}) - f(OPT_i))) 
    \leq \frac{1}{2} \mathbb{E}[a_i] + 5 rad.
\end{equation}
Note that $OPT_i = \left(OPT \cup X_i\right) \cap Y_i = OPT_{i-1} \cup \left\{u_i\right\}$. Further, by Corollary \ref{corollary4rad}, under the clean event, for every $1 \leq i \leq n$, $\bar{a}_i + \bar{b}_i \geq -4 rad$. As $\bar{b}_i < 0$, then,  $\bar{a}_i \geq -4 rad$, we have 
\begin{equation} 
    \frac{\bar{a}_i}{2} + 2 rad\geq 0>\bar{b}_i. \label{eq:prf:case3:10}
\end{equation}
If $u_i \in OPT$, then $OPT_i = OPT_{i-1}$.  Thus, we have 
\begin{align}
    &\hspace{-1cm}\mathbb{E}[f(OPT_i) - f(OPT_{i-1})] \nonumber\\
    &= 0 \nonumber\\
    &\leq \frac{\bar{a}_i}{2} + 2rad \tag{by \eqref{eq:prf:case3:10}}\\
    &\leq \frac{1}{2} \mathbb{E}[a_i]+ 3rad \tag{using (\ref{ai_bi_bounds})}\\
    &\leq \frac{1}{2} \mathbb{E}[a_i] + 5rad. \nonumber
\end{align}

If $u_i \notin OPT$, 
then $OPT_{i-1} \subseteq Y_{i-1}$ and $\left(Y_{i-1} \backslash\left\{u_i\right\}) \cup\left\{u_i\right\}=Y_{i-1}\right.$  Thus, like in case 1, \eqref{eq:prf:case1:14} holds.  Negating \eqref{eq:prf:case1:14}, we obtain,
\begin{align}
    &\hspace{-0.5cm}\mathbb{E}f(OPT_{i-1}) - \mathbb{E}f(OPT_{i-1} \cup \left\{u_i\right\}) \nonumber\\
    &\leq  \mathbb{E} f\left(Y_{i-1} \backslash \left\{u_i\right\}\right)  - \mathbb{E} f\left(Y_{i-1}\right) \tag{from \eqref{eq:prf:case1:14}}\\
    &= \mathbb{E}[b_i] \tag{from def. \eqref{eq:abdefn}} \\
    &\leq \bar{b}_i + 2rad \tag{using (\ref{ai_bi_bounds})}\\
    &\leq \frac{\bar{a}_i}{2} + 4rad  \tag{by \eqref{eq:prf:case3:10}} \\
    &\leq \frac{1}{2}\mathbb{E}[\bar{a}_i] + 5rad. \nonumber
\end{align}

$\quad$ \textbf{Case 4} ($\bar{a}_i \geq 0$ and $\bar{b}_i > 0$): In this case, $a_i^{\prime} = \bar{a}_i$ and $b_i^{\prime} = b_i$. Hence, by the algorithm with probability $\frac{a_i{\prime}}{a_i{\prime}+b_i{\prime}}$, $X_i \leftarrow X_{i-1} \cup \left\{u_i\right\}$ and $Y_i \leftarrow Y_{i-1}$, and with probability $\frac{b_i^{\prime}}{a_i^{\prime} + b_i^{\prime}}$, $Y_i \leftarrow Y_{i-1} \setminus \left\{u_i\right\}$ and $X_i \leftarrow X_{i-1}$. We have,

\begin{align}
&\hspace{-.25cm}\mathbb{E}\left[f\left(X_i\right)-f\left(X_{i-1}\right)+f\left(Y_i\right)-f\left(Y_{i-1}\right)\right] \nonumber \\
&=\mathbb{E} [\mathbb{E}\left[f\left(X_i\right)-f\left(X_{i-1}\right)+f\left(Y_i\right)-f\left(Y_{i-1}\right)\right| \bar{a}_i, \bar{b}_i]] \nonumber \\
&=\mathbb{E}[\frac{\bar{a}_i}{\bar{a}_i+\bar{b}_i} \mathbb{E}\left[f\left(X_{i-1} \cup \left\{u_i\right\}\right)-f\left(X_{i-1}\right)\right] \nonumber \\
&\quad + \frac{\bar{b}_i}{\bar{a}_i+\bar{b}_i} \mathbb{E} \left[f\left(Y_{i-1} \backslash \left\{u_i\right\} \right)-f\left(Y_{i-1}\right)\right]] \nonumber \\
&= \mathbb{E}\left[\frac{\bar{a}_i\mathbb{E}[a_i]}{\bar{a}_i+\bar{b}_i}+\frac{\bar{b}_i \mathbb{E} [b_i]}{\bar{a}_i+\bar{b}_i}\right]  \tag{def. of $a_i$ and $b_i$}\\
&\geq \mathbb{E}\left[\frac{\bar{a}_i(\bar{a}_i - 2 rad)}{\bar{a}_i+\bar{b}_i}+\frac{\bar{b}_i (\bar{b}_i - 2 rad)}{\bar{a}_i+\bar{b}_i}\right] \tag{using (\ref{ai_bi_bounds})}\\
&= \mathbb{E}\left[\frac{\bar{a}_i^2}{\bar{a}_i+\bar{b}_i}+\frac{\bar{b}_i^2 }{\bar{a}_i+\bar{b}_i}- \frac{2rad(\bar{a}_i+\bar{b}_i) }{\bar{a}_i+\bar{b}_i}\right] \nonumber\\
%
&=\mathbb{E}\left[\frac{\bar{a}_i^2+\bar{b}_i^2}{\bar{a}_i+\bar{b}_i}\right] - 2 rad. \nonumber 
\end{align}

Hence, 
\begin{equation}
\label{ai2+bi2}
\begin{aligned}
& \frac{1}{2} \mathbb{E}\left[\frac{\bar{a}_i^2+\bar{b}_i^2}{\bar{a}_i+\bar{b}_i}\right] -  rad\\
& \leq \frac{1}{2} \mathbb{E}\left[f\left(X_i\right)-f\left(X_{i-1}\right)+f\left(Y_i\right)-f\left(Y_{i-1}\right)\right] .  
\end{aligned}
\end{equation}
Moreover, 
\begin{equation*}
\begin{aligned}
&\mathbb{E} [f(OPT_{i-1}) - f(OPT_{i})] \\
&= \mathbb{E}[  \frac{\bar{a}_i}{\bar{a}_i+\bar{b}_i} \mathbb{E} [f(OPT_{i-1}) - f(OPT_{i-1}\cup \left\{u_i\right\})]\\
&+  \frac{\bar{b}_i}{\bar{a}_i+\bar{b}_i}\mathbb{E} [f(OPT_{i-1}) - f(OPT_{i-1} \setminus \left\{u_i\right\})]]    
\end{aligned}
\end{equation*}

If $u_i \notin OPT \Rightarrow $ second term is zero and  $OPT_{i-1} \subseteq Y_{i-1} \setminus \left\{u_i\right\}$. Thus, by submodularity in expectation, 
\begin{equation*}
    \begin{aligned}
    &\hspace{-1cm}\mathbb{E} [f(OPT_{i-1}) - f(OPT_{i-1}\cup \left\{u_i\right\})]\\
    &\leq \mathbb{E} [f(Y_{i-1} \setminus \left\{u_i\right\}) - f(Y_i)] \\
    &= \mathbb{E}[b_i]\\
    &\leq \bar{b}_i + 2 rad,
    \end{aligned}
\end{equation*}
so if $u_i \notin OPT$ then
\begin{align}
&\hspace{-1cm}\mathbb{E} [f(OPT_{i-1}) - f(OPT_{i})] \nonumber\\
&\leq \mathbb{E}[  \frac{\bar{a}_i}{\bar{a}_i+\bar{b}_i}(\bar{b}_i + 2 rad)+  \frac{\bar{b}_i}{\bar{a}_i+\bar{b}_i}0] \nonumber\\
&=\mathbb{E}[ \frac{\bar{a}_i \bar{b}_i}{\bar{a}_i+\bar{b}_i}]  + 2 rad\mathbb{E}[ \frac{\bar{a}_i}{\bar{a}_i+\bar{b}_i}] . \label{eq:case4:ui-notin-OPT} 
\end{align}

If $u_i \in OPT \Rightarrow $ first term is zero and $X_{i-1} \subseteq\left(OPT\cup X_{i-1}\right) \cap Y_{i-1} \setminus \left\{u_i\right\}$.  Hence, by submodularity in expectation, we have
\begin{equation*}
    \begin{aligned}
    &\hspace{-1cm}\mathbb{E} [f(OPT_{i-1}) - f(OPT_{i-1}\setminus \left\{u_i\right\})] \\
    &\leq \mathbb{E} [f(X_{i-1} \setminus \left\{u_i\right\}) - f(X_i)] \\
    & = \mathbb{E}[a_i] \\
    &\leq \bar{a}_i + 2 rad .
    \end{aligned}
\end{equation*} 

Thus,  if $u_i \in OPT$ then
\begin{align}
&\hspace{-1cm}\mathbb{E} [f(OPT_{i-1}) - f(OPT_{i})] \nonumber\\
&\leq \mathbb{E}[  \frac{\bar{a}_i}{\bar{a}_i+\bar{b}_i}0+  \frac{\bar{b}_i}{\bar{a}_i+\bar{b}_i}(\bar{a}_i + 2 rad)] \nonumber\\
&=\mathbb{E}[ \frac{\bar{a}_i \bar{b}_i}{\bar{a}_i+\bar{b}_i}]  + 2 rad\mathbb{E}[ \frac{\bar{b}_i}{\bar{a}_i+\bar{b}_i}] . \label{eq:case4:ui-in-OPT} 
\end{align}

Since we are conditioning on ($\bar{a}_i \geq 0$ and $\bar{b}_i > 0$) for this case, then we have that
\begin{align*}
    \mathbb{E}[ \frac{\bar{a}_i}{\bar{a}_i+\bar{b}_i}] \geq 0 \qquad \text{and} \qquad \mathbb{E}[ \frac{\bar{b}_i}{\bar{a}_i+\bar{b}_i}] \geq 0
\end{align*}
Combining the bounds \eqref{eq:case4:ui-notin-OPT} and \eqref{eq:case4:ui-in-OPT}, we have
\begin{align}
&\hspace{-.2cm}\mathbb{E} [f(OPT_{i-1}) - f(OPT_{i})] \nonumber\\
&\leq\mathbb{E}[ \frac{\bar{a}_i \bar{b}_i}{\bar{a}_i+\bar{b}_i}]  + 2 rad\mathbb{E}[ \frac{\bar{b}_i}{\bar{a}_i+\bar{b}_i}] + 2 rad\mathbb{E}[ \frac{\bar{b}_i}{\bar{a}_i+\bar{b}_i}] \nonumber\\%
&=\mathbb{E}[ \frac{\bar{a}_i \bar{b}_i}{\bar{a}_i+\bar{b}_i}]  + 2 rad, \label{eq:case4:combined} 
\end{align}
which holds regardless of $u_i$'s membership in $OPT$.

For $x+y > 0$, by the Cauchy-Schwarz inequality, 
\begin{equation}
\label{cauchy}
    \frac{x y}{x + y} \leq \frac{1}{2} \frac{x^{2} + y^{2}}{x + y}. 
\end{equation} 

Combining the above observations, it follows that
\begin{equation*}
    \begin{aligned}
    &\mathbb{E} [f(OPT_{i-1}) - f(OPT_{i})] \\
    &\stackrel{\eqref{eq:case4:combined} }{\leq} \mathbb{E}[ \frac{\bar{a}_i \bar{b}_i}{\bar{a}_i+\bar{b}_i}]  + 2 rad \\
    &\stackrel{(\ref{cauchy})}{\leq} \frac{1}{2} \mathbb{E}[ \frac{\bar{a}_i^{2} + \bar{b}_i^{2}}{\bar{a}_i+\bar{b}_i}]  + 2 rad \\
    &\stackrel{(\ref{ai2+bi2})}{\leq} \frac{1}{2} \mathbb{E}\left[f\left(X_i\right)-f\left(X_{i-1}\right)+f\left(Y_i\right)-f\left(Y_{i-1}\right)\right] + 3 rad \\
    &\leq \frac{1}{2} \mathbb{E}\left[f\left(X_i\right)-f\left(X_{i-1}\right)+f\left(Y_i\right)-f\left(Y_{i-1}\right)\right] + 5 rad.
    \end{aligned}
\end{equation*}
\end{proof}

\begin{corollary} 
\label{corollary 5nrad}
Under the clean event $\mathcal{E}$,
\begin{equation}
    \begin{aligned}
    \mathbb{E}[f(X_n)] + \frac{5}{2} n \, rad
     \geq \frac{1}{2} \mathbb{E}[f(OPT)].
    \end{aligned}
\end{equation}
\end{corollary}
\begin{proof}
Summing up \eqref{ineq:lemma 5rad} in Lemma \ref{lemma 5rad} for $1 \leq i \leq n$ yields,  
\begin{equation*}
\begin{aligned}
&\hspace{-1cm} \sum_{i=1}^n\mathbb{E} \left[f\left(O P T_{i-1}\right)-f\left(O P T_i\right)\right] \\ 
&\leq 5 n \, rad + \frac{1}{2}  \sum_{i=1}^n \mathbb{E} \left[f\left(X_i\right)-f\left(X_{i-1}\right)\right] \\
&\qquad + \frac{1}{2}  \sum_{i=1}^n \mathbb{E}  \left[f\left(Y_i\right)-f\left(Y_{i-1}\right)\right].
\end{aligned}
\end{equation*}
Notice the sums above are telescopic. Simplifying them,
\begin{equation*}
\begin{aligned}
&\mathbb{E}\left[f\left(O P T_0\right)-f\left(O P T_n\right)\right]  \\
&\leq 5nrad + \frac{1}{2} \cdot \mathbb{E}\left[f\left(X_n\right)-f\left(X_0\right)+f\left(Y_n\right)-f\left(Y_0\right)\right]  \\
&\leq 5nrad + \frac{\mathbb{E}\left[f\left(X_n\right)+f\left(Y_n\right)\right]}{2} .
\end{aligned}
\end{equation*}
We obtain the result by noticing that $OPT_0 = OPT$ and $OPT_n = X_n = Y_n$.
\end{proof}

Having discussed the key Lemmas, the next result provides the bound on expected cumulative $\frac{1}{2}$-regret of RGL.

\begin{figure*}[t]
\begin{center}
\includegraphics[width=13.5cm]{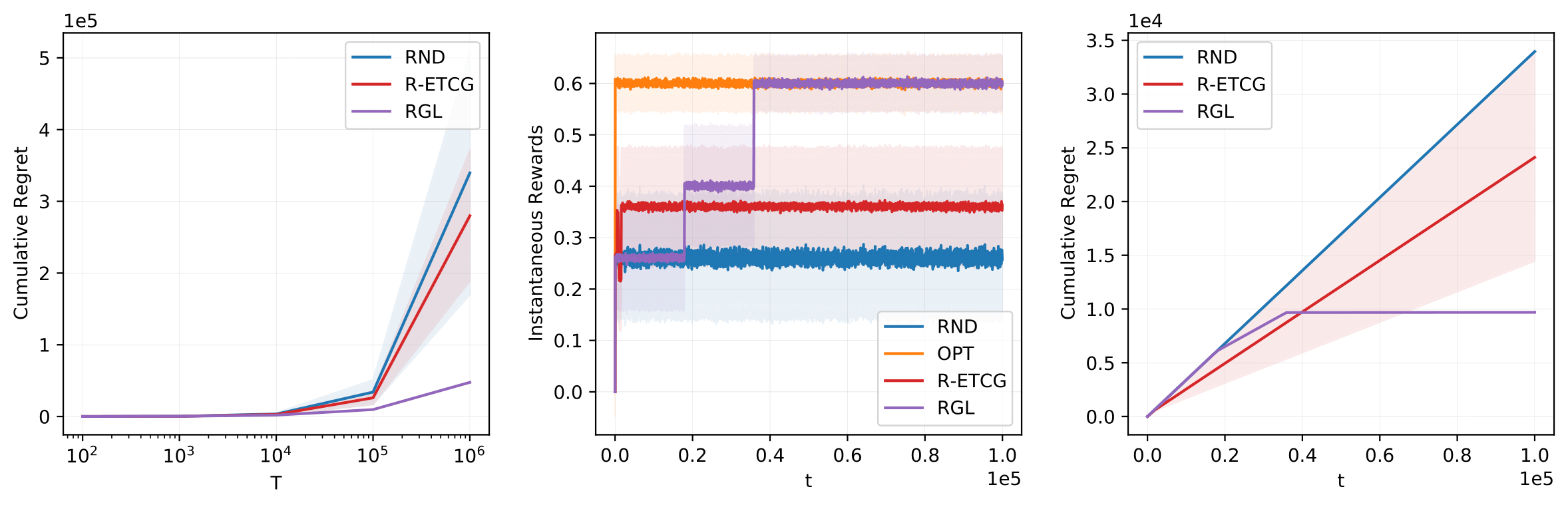}
\end{center}
\vspace{-.2in}
\caption{\small Comparison results for the non-monotone stochastic submodular reward function. From left to right, the plots show cumulative regret as a function of time step $T$, instantaneous rewards as a function of time step $t$, and cumulative regret as a function of time horizon $t$, respectively.}
\label{figure1}
\end{figure*}

\begin{figure*}[t]
\begin{center}
\includegraphics[width=13.5cm]{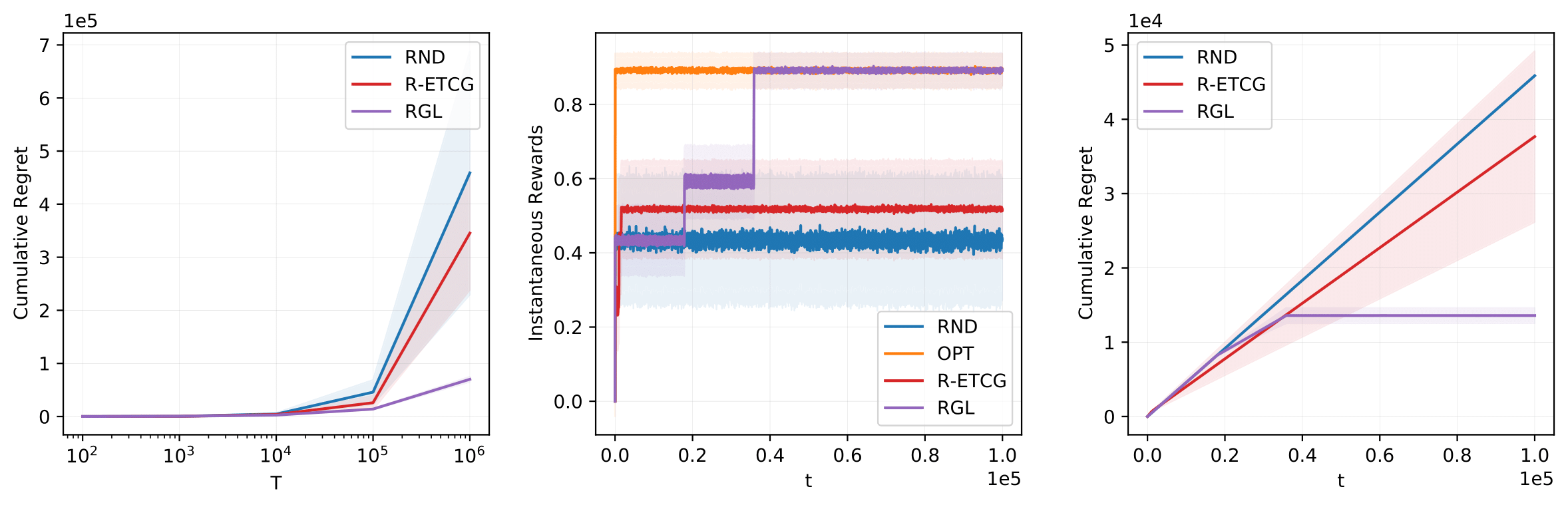}
\end{center}
\vspace{-.2in}
\caption{\small Comparison results for the non-monotone stochastic non-submodular reward function. From left to right, the plots show cumulative regret as a function of time step $T$, instantaneous rewards as a function of time step $t$, and cumulative regret as a function of time horizon $t$, respectively.}
\label{figure2}
\end{figure*}

\begin{theorem}
 For the sequential decision making problem defined in Section \ref{problem} with $T \geq 2$, the expected cumulative $\frac{1}{2}$-regret of RGL is at most $\mathcal{O}(n T^{\frac{2}{3}} \log(T)^{\frac{1}{3}})$.
\end{theorem}

\begin{proof}

We first condition the expected cumulative regret on the clean event. 
\begin{equation}
\label{regretboundclean}
\begin{aligned}
&\hspace{-.5cm}\mathbb{E}(R(T)|\mathcal{E})\\
&=\frac{1}{2} T \mathbb{E}[ f\left(OPT\right)]-\sum_{t=1}^T \mathbb{E}[ f\left(S_{t}\right)]\\
&=\sum_{t=1}^T\left(\frac{1}{2} \mathbb{E} [f(OPT)]-\mathbb{E}[ f\left(S_t\right)]\right)\\
&=\sum_{i=1}^n  \sum_{j=1}^{m} \biggr[\left(\frac{1}{2} \mathbb{E}[f(OPT)]-\mathbb{E}[f\left(X_{i-1}\right)]\right)  \\
&+ \left(\frac{1}{2} \mathbb{E}[f(OPT)]-\mathbb{E} [f\left(X_{i-1} \cup \left\{u_i\right\}\right)]\right) \\
&+ \left(\frac{1}{2} \mathbb{E}[f(OPT)]-\mathbb{E}[f\left(Y_{i-1}\right)]\right)  \\
&+ \left(\frac{1}{2} \mathbb{E}[f(OPT)]-\mathbb{E}[f\left(Y_{i-1} \setminus \left\{u_i\right\}\right)]\right)\biggr] \\
&+\sum_{t=4 n m + 1}^T\left(\frac{1}{2} \mathbb{E}[f(OPT)]-\mathbb{E} [f\left(S_t\right)]\right).
\end{aligned}    
\end{equation}
  
We split the sum into two parts, the first accounting for cumulative regret incurred during the exploration phase and the second for the exploitation phase. 
During exploration, for each arm $u_i$ 
the agent 
plays four subsets, $X_i$, $Y_i$, $X_{i-1} \cup \left\{u_i\right\}$, and $Y_{i-1} \setminus \left\{u_i\right\}$, for $m$ times each. Hence, the agent explores for 
$4 m n$ time steps. Since $f(\cdot)$ is bounded in $[0,1]$, for any subset $S_t$ played at time $t$ by the agent, 
\begin{equation}
\label{eq:boundhalf}
     \frac{1}{2} \mathbb{E}[f\left(OPT\right)]-\mathbb{E}[ f\left(S_t\right)] \leq \frac{1}{2} .\\
\end{equation}

Substituting (\ref{eq:boundhalf}) in (\ref{regretboundclean}), we have
$$
\begin{aligned}
&\hspace{-.5cm}\mathbb{E}[R(T) \mid \mathcal{E}] \\
& \leq 4 n m \frac{1}{2}+\sum_{t=T_{n+1}}^T\left(\frac{1}{2} \mathbb{E}[f(OPT)]-\mathbb{E} [f\left(S_t\right)]\right) \\
&=2 n m+\sum_{t=T_{n+1}}^T\left(\frac{1}{2} \mathbb{E}[f(OPT)]-\mathbb{E}[f\left(X_n\right)]\right) .
\end{aligned}
$$

From Corollary \ref{corollary 5nrad}, we have $\frac{1}{2} \mathbb{E}(f(OPT)) - \mathbb{E}[f(X_n)] \leq \frac{5}{2} n \, rad$. Thus, 
$$
\begin{aligned}
\mathbb{E}[R(T) \mid \mathcal{E}] & \leq 2 n m+\sum_{t=T_{n+1}}^T\left(\frac{5}{2} n \, rad \right) \\ & \leq 2 n m+\frac{5}{2} T n \, rad.
\end{aligned}
$$

Since $rad = \sqrt{2\log(T)/m}$, we have 
$$
\begin{aligned}
\mathbb{E}(R(T) \mid \mathcal{E})  & \leq 2 n m+\frac{5}{2} T n \sqrt{2\frac{\log(T)}{m}}.
\end{aligned}
$$

The above inequality is true for all $m$ strictly greater than zero. Hence, to find a tighter bound, we find $m^{*}$ that minimizes the left side. The exact minimizer is:
$$
m^{\star} = \left(T \sqrt{\frac{25}{32} \log (T)}\right)^{2 / 3}.
$$ 
Therefore, we choose $m = \lceil m^{\star}\rceil$.
$$
\begin{aligned}
\mathbb{E}(R(T) \mid \mathcal{E}) &\leq 2 n\lceil m^{\star}\rceil+ \frac{5}{2} n T \sqrt{\frac{2 \log (T)}{ \lceil m^{\star}\rceil}} \\
& \leq 2 n\lceil m^{\star}\rceil+ \frac{5}{2} n T \sqrt{\frac{2 \log (T)}{m^*}} \\
\end{aligned}
$$

For $T\geq 2$, $m^{\star} \geq \frac{1}{2}$ and thus $\lceil m^{\star}\rceil \leq 2 m^{*}$. Thus, we have
%
%
$$
\begin{aligned}
\mathbb{E}(R(T) \mid \mathcal{E}) & \leq 4 n m^{\star} + \frac{5}{2} n T \sqrt{\frac{2 \log (T)}{m^*}} \\
& \leq \mathcal{O}(n T^{\frac{2}{3}} \log(T)^{\frac{1}{3}})
\end{aligned}
$$

Under the bad event, i.e., the complement $\bar{\mathcal{E}}$ of the good event $\mathcal{E}$, given that the rewards are bounded in $[0,1]$, it can be easily seen that $\mathbb{E}(R(T) \mid \mathcal{\bar{E}})  \leq T$. Moreover, by using Lemma \ref{hoeffding} in Appendix \ref{missing proofs}, the Hoeffding inequality \citep{hoeffding1994probability}, we have $\mathbb{P}(\mathcal{\bar{E}})  \leq \frac{8n}{T^{4}}$, see Lemma \ref{lemma4} in Appendix \ref{missing proofs}. Therefore, we obtain $\mathbb{E}(R(T)) \leq \mathcal{O}(n T^{\frac{2}{3}} \log(T)^{\frac{1}{3}}).$
\end{proof}

\begin{remark}
When the time horizon $T$ is not known, we can extend our result to an anytime algorithm using the geometric doubling trick. Essentially, we pick a geometric sequence $T_i=T_0 2^i$ for $i \in\{1,2, \cdots\}$, where $T_0$ is a large enough number to let the algorithm initialize, and run RGL within time interval $T_{i+1}-T_i$ with a full restart, \citep{besson2018doubling}. From Theorem 4 in the work of \cite{besson2018doubling}, it follows that the regret bound conserves the original $T^{2 / 3} \log (T)^{1 / 3}$ dependence with only changes in constant factors.
\end{remark}

\section{Experiments}

In this section, we empirically evaluate our RGL algorithm in non-monotone, submodular and non-submodular settings. For further experiments, we refer the reader to the linear reward minus cost experiment in Appendix \ref{linear_minus_cost}, and to the revenue maximization over social networks experiment in Appendix \ref{revenue_maximizatio_social}.

We compare our method to the exact optimal solution and compute the empirical mean over different repetitions of the cumulative full regret instead of the cumulative $\frac{1}{2}$-regret, defined as follows, 
\vspace{-.1in}
$$
\bar{\mathcal{R}}(T) = \frac{1}{rep} \sum_{n=1}^{rep} \sum_{t=1}^T ( f\left(OPT\right)- f\left(S_{t}\right)).
$$
We test the algorithms on a non-monotone stochastic submodular function of the chosen set $S$, defined as 
$f(S) = \min(\max(g(S) + \varepsilon,0),1)$, 
where $\varepsilon \sim \mathcal{N}(\mu,\sigma)$. In our experiments, we choose a non-monotone submodular example of $g(S)$, where $g(\left\{\right\}) = 0.2$, $g(\left\{1\right\}) = 0$, $g(\left\{2\right\}) = 0.6$, $g(\left\{1,2\right\}) = 0.2$. Note that $\mathbb{E}[f(S)] = g(S)$, and $g(S)$ is submodular. 

In the second experiment, we choose a non-monotone non-submodular example of $g(S)$, where $g(\left\{\right\}) = 0.3$, $g(\left\{1\right\}) = 0$, $g(\left\{2\right\}) = 0.5$, $g(\left\{1,2\right\}) = 0.9$. Notice, that $\mathbb{E}[f(S)] = g(S)$, and $g(S)$ is not submodular. 

We run our method for $T \in \left\{10^{2}, 10^{3},  10^{4} ,10^{5}, 10^{6}\right\}$ time horizons. We assume $\varepsilon \sim \mathcal{N}(0,0.1)$. We average our experiments over $rep = 20$ repetitions. We average the instantaneous rewards over a window of size 50.

We use the optimal solution, which is $\left\{2\right\}$ in the first experiment and $\left\{1, 2\right\}$ in the second experiment, and run it in the online setting, where the optimal agent (OPT) only exploits the best set of arms throughout the time until $T$, see Algorithm \ref{alg:OPT}. Moreover, we compare to random bandits (RND), see Algorithm \ref{alg:RND}, which at each time step plays a random subset of $\Omega$, where each arm is sampled independently with probability $\frac{1}{2}$. The random algorithm in the offline setting has $\frac{1}{4}$-approximation guarantee \citep{feige2011maximizing}. Furthermore, we compare to one online monotone submodular maximization algorithm, ETCG, \citep{nie2022explore}. Unfortunately, the online algorithms for monotone submodular maximization require an extra input, which is the cardinality $k$. Thus, we define R-ETCG, see Algorithm \ref{alg:cap}, which initially generates a random $k \sim \mathcal{U}(0,n)$, then finds the best $k$ arms. 

From Fig. \ref{figure1} for the sub-modular function case, it can be seen that RGL reaches the optimum. From Fig. \ref{figure2}, in the non-submodular case, it can be seen that RGL still reaches the optimum. In both experiments, RGL outperforms all the above-defined benchmarks. Even though the theory is not developed for non-submodular cases, the approach can still work well even in such cases. Further, the proposed algorithm outperforms R-ETCG, indicating that the algorithms for monotone functions cannot be directly applied to the non-monotone case. 

\begin{remark}
The cumulative regret upper bound dependence of $O(T^{2/3})$ is on the horizon $T$ (not time-step $t$) (see left sub-figures in all Figures, which have cumulative regret curves increasing in $T$). For a fixed time horizon T, RGL found the optimal set of arms, which makes its cumulative regret for a fixed time horizon T a constant w.r.t. time $t$ (right sub-figures for Fig. \ref{figure1} and \ref{figure2}). Furthermore, the theoretical guarantees are for the worst-case scenario, i.e., the theory gives an upper bound on the regret, which for some instances, will be lower.
\end{remark}

\section{Conclusion}

This paper proposes RGL, the first online stochastic non-monotone submodular maximization algorithm under full-bandit feedback, i.e. when the agent only receives the reward for a chosen set of arms and has no extra information about the individual arms. The proposed algorithm provably achieves a $\frac{1}{2}$-regret upper bound of $\Tilde{\mathcal{O}}(n T^{\frac{2}{3}})$ for horizon $T$ and number of arms $n$. Moreover, the algorithm empirically outperforms the considered baselines under full-bandit feedback. 

We note that the existing results for sub-modular bandits with full-bandit feedback also achieve $\Tilde{\mathcal{O}}( T^{\frac{2}{3}})$ regret bound in the monotone function setup \citep{nie2022explore,49310}. Further, the results for non-monotone function in adversarial setting is also $\Tilde{\mathcal{O}}( T^{\frac{2}{3}})$ \citep{49310}. While a formal lower bound for this setup has not been studied, proving such a lower bound or improving the regret bounds in all these setups is an open problem.

\section{Acknowledgement}
This work was supported in part by the National Science Foundation under Grants 2149588 and 2149617.

\bibliographystyle{plainnat}

\bibliography{main}

\appendix
\onecolumn

\section{Motivating Examples for (Non-montone) Submodular Maximization}
\label{appendix_motivation}

\subsection{Data Summarization}
As huge amount of data is generated daily, selecting a good representative subset of data points remains as a challenge. Often, the utility function capturing the coverage or diversity of a subset of the entire dataset satisfies submodularity \citep{mirzasoleiman2016fast}. However, utility functions that accommodate diversity are not necessarily monotone as they penalize larger solutions \citep{tschiatschek2014learning, dasgupta-etal-2013-summarization}.

\subsection{Feature Selection}
One compelling use of non-monotone submodular maximization algorithms is modeling some learning problems such as feature selection \citep{das2008algorithms, khanna2017scalable, elenberg2018restricted, qian2019fast}. Optimizing feature selection can be modeled as a non-monotone submodular maximization due to the possible overfitting to the training data \citep{fahrbach2018non}.

\subsection{Recommender Systems}
Recommending items with redundant information leads to diminishing returns on utility. This problem of sequentially recommending
sets of items to users has been studied through the framework of contextual submodular combinatorial bandits \citep{qin2013promoting, takemori2020submodular}. The optimization is not necessarily monotone as adding further recommendations might lead to a counter effect \citep{amanatidis2020fast}. 

\subsection{Influence Maximization}
\label{inf-max}
One possible way to market a newly developed product can be done by selecting a set of highly influential people and hope they recommend it to their communities. A recent line of research has considered the problem as a multi-armed bandit problem (with extra feedback) without requiring the knowledge of the network and diffusion model \citep{lei2015online, wen2017online, vaswani2017model, li2020online, perrault2020budgeted}.  Most works consider that there is a fixed constraint on cardinality or budget.  However, a revenue maximization model to maximize income from influence minus the costs is in general a non-monotone unconstrained submodular maximization problem \citep{lu2012profit}.

\section{Additional Lemmas and Proofs}
\label{missing proofs}

\subsection{Probability of the Clean Event}

Hoeffding’s inequality \citep{hoeffding1994probability} is a powerful technique for bounding probabilities of bounded random variables. We state the inequality, then we use it to show that $\mathcal{E}$ happens with high probability.

\begin{lemma}
\label{hoeffding}
\textbf{(Hoeffding's inequality).} Let $X_1, X_2, ... , X_n$ be independent random variable bounded in $[0,1]$ and let $\bar{X}$ their empirical mean. Then we have for any $\varepsilon > 0$, 
\begin{equation}\nonumber
    \mathbb{P}(|\bar{X} - \mathbb{E}(X)| \geq \varepsilon) \leq 2 \exp \left(- 2 n \varepsilon^{2}\right).
\end{equation}
\end{lemma}

\begin{lemma}
\label{lemma4}
The probability of clean event $\mathcal{E}$ satisfies
\begin{equation}
    \mathbb{P}(\mathcal{E}) \geq 1 - \frac{8 n}{T^4}.
\end{equation}
\end{lemma}

\begin{proof}
Applying  \cref{hoeffding} to the empirical mean $\bar{f}\left(S\right)$ of $m$ rewards for action $S$ and choosing $\epsilon=\operatorname{rad}=$ $\sqrt{2 \log (T) / m}$, we have

\begin{align}
\mathbb{P}\left[\left|\bar{f}\left(S \right)-f\left(S\right)\right| \geq \operatorname{rad}\right] & \leq 2 \exp \left(-2 m \operatorname{rad}^2\right) \tag{by \cref{hoeffding}} \\
&=2 \exp (-2 m(2 \log (T) / m)) \nonumber \\
&=2 \exp (-4 \log (T)) \nonumber \\
&=\frac{2}{T^4}.
\end{align}

For each arm $u_i$, the agent plays the following list of actions $\mathcal{S}_i=\left[X_{i-1}, X_{i-1} \cup\left\{u_i\right\}, Y_{i-1}, Y_{i-1} \backslash\left\{u_i\right\}\right]$ exactly $m$ times, then computes marginal gain estimates. Thus, for any individual action $S \in \mathcal{S}_i$, we can bound the probability that its sample mean $\bar{f}\left(S\right)$ is within a specified confidence radius (complementary of the event above) as

\begin{align}
\forall S \in \mathcal{S}_i \quad \mathbb{P}\left[\left|\bar{f}\left(S\right)-f\left(S\right)\right|<\operatorname{rad}\right] &=1-\mathbb{P}\left[\left|\bar{f}\left(S\right)-f\left(S\right)\right| \geq \operatorname{rad}\right] \nonumber \\
& \geq 1-\frac{2}{T^4}. \label{lowerbound}
\end{align}

We now focus on bounding $\mathbb{P}\left(\mathcal{E}_i \mid X_{i-1} = X, Y_{i-1} = Y\right)$. By conditioning on the sets decided in the previous phase, $X_{i-1} = X, Y_{i-1} = Y$, we know all the actions that will be played in the current phase $i$, i.e. $\mathcal{S}_i$. The rewards of all the actions are bounded in $[0,1]$ and are conditionally independent (given the corresponding action).

\begin{align}
\mathbb{P}\left(\mathcal{E}_i  \mid X_{i-1} = X, Y_{i-1} = Y\right) & =\mathbb{P} \left(\bigcap_{S \in \mathcal{S}_i}\left\{\left|\bar{f}\left(S\right)- \mathbb{E} \left[f\left(S\right)\right]\right|<\operatorname{rad}\right\} \mid X_{i-1} = X, Y_{i-1} = Y\right) \nonumber \tag{by \eqref{concentration}}\\
&=\prod_{S \in \mathcal{S}_i} \mathbb{P}\left(\left\{\left|\bar{f}\left(S\right)-f\left(S\right)\right|<\operatorname{rad}\right\} \mid X_{i-1} = X, Y_{i-1} = Y\right) \nonumber \tag{rewards are independent when conditioned on actions} \\
& \geq\left(1-\frac{2}{T^4}\right)^{\left|\mathcal{S}_i\right|} \nonumber \tag{by \eqref{lowerbound}}\\
&=\left(1-\frac{2}{T^4}\right)^{4} \label{eq:pow4} 
\end{align}

With this, we can then lower bound the probability of the clean event $\mathcal{E}$,
\begin{align}
\mathbb{P}(\mathcal{E}) &=\mathbb{P}\left(\mathcal{E}_1 \cap \cdots \cap \mathcal{E}_n\right) \nonumber \tag{by \eqref{Eq:bigE}}\\
&=\prod_{i=1}^n \mathbb{P}\left(\mathcal{E}_i \mid \mathcal{E}_1, \ldots, \mathcal{E}_{i-1}\right) \nonumber\\
&= \prod_{i=1}^n \sum_{X, Y} \mathbb{P}\left(X_{i-1} = X, Y_{i-1} = Y, \mathcal{E}_i  \mid \mathcal{E}_1, \ldots, \mathcal{E}_{i-1}\right) \nonumber \tag{law of total probability}\\
&= \prod_{i=1}^n \sum_{X, Y} \mathbb{P}\left(X_{i-1} = X, Y_{i-1} = Y \mid \mathcal{E}_1, \ldots, \mathcal{E}_{i-1}\right) \times \mathbb{P}\left(\mathcal{E}_i \mid X_{i-1} = X, Y_{i-1} = Y, \mathcal{E}_1, \ldots, \mathcal{E}_{i-1}\right)\nonumber\\
&= \prod_{i=1}^n \sum_{X, Y} \mathbb{P}\left(X_{i-1} = X, Y_{i-1} = Y \mid \mathcal{E}_1, \ldots, \mathcal{E}_{i-1}\right) \times \mathbb{P}\left(\mathcal{E}_i \mid X_{i-1} = X, Y_{i-1} = Y\right) \nonumber\\
&\geq \prod_{i=1}^n \sum_{X, Y} \mathbb{P}\left(X_{i-1} = X, Y_{i-1} = Y \mid \mathcal{E}_1, \ldots, \mathcal{E}_{i-1}\right) \times \left(1-\frac{2}{T^4}\right)^{4}\nonumber \tag{by \eqref{eq:pow4}}\\
&= \prod_{i=1}^n \left(1-\frac{2}{T^4}\right)^{4} \sum_{X, Y} \mathbb{P}\left(X_{i-1} = X, Y_{i-1} = Y \mid \mathcal{E}_1, \ldots, \mathcal{E}_{i-1}\right) \nonumber\\
&= \prod_{i=1}^n \left(1-\frac{2}{T^4}\right)^{4} \nonumber\\
&= \left(1-\frac{2}{T^4}\right)^{4n} \nonumber\\
&\geq \left(1-\frac{8n}{T^4}\right).\nonumber \tag{Bernoulli's inequality}
\end{align}
\end{proof}

\subsection{Proof of Case 2 in Lemma 2.}
\label{lemma_2_case_2}

It is sufficient to prove the inequality conditioned on any event of the form $X_{i-1} = S_{i-1}$ where $S_{i-1} \subseteq \left\{u_1, ..., u_{i-1} \right\}$, for which the probability $X_{i-1} = S_{i-1}$ is non-zero.  
The remainder of the proof assumes everything is conditioned on this event. The proof of Lemma 2 was divided in 4 cases in the text, where the detailed proof of three of them is provided in the main text. The proof of Case 2 is provided here for completeness. 


\begin{proof}

$\quad$ \textbf{Case 2} ($\bar{a}_i < 0$ and $\bar{b}_i\geq 0$):  In this case $\bar{a}_i\leq 0 \Rightarrow a_i^{\prime} = 0 \Rightarrow \frac{b_i^{\prime}}{a_i^{\prime}+ b_i^{\prime}} = 1$. Thus,  $X_i = X_{i-1}$ and $Y_{i} = Y_{i-1} \setminus \left\{u_i\right\}$.  Since $Y_{i} = Y_{i-1} \setminus \left\{u_i\right\}$, we have 
\begin{equation}
\begin{aligned}
\mathbb{E}[b_i] &= \mathbb{E}[f(Y_{i-1} \setminus \left\{u_i\right\}) - f(Y_{i-1})] \\
&= \mathbb{E}[f(Y_{i}) - f(Y_{i-1})].
\end{aligned}\label{eq:eai:case2}
\end{equation}
Since $X_i = X_{i-1}$, the relation \eqref{ineq:lemma 5rad} that we want to show  reduces to,
\begin{equation*}
    \begin{aligned}
    &\hspace{-1cm}\mathbb{E}[f(OPT_{i-1}) - f(OPT_i)]  \leq \frac{1}{2} \mathbb{E}[f(Y_i) - f(Y_{i-1})]+ 5 rad .
    \end{aligned}
\end{equation*}
Note that 
\begin{equation} \label{OPT_i}
    OPT_i = \left(OPT \cup X_i\right) \cap Y_i = OPT_{i-1} \setminus \left\{u_i\right\}
\end{equation}

If $u_i \notin OPT \Rightarrow OPT_i = OPT_{i-1}$. Thus, 
\begin{align}
    \mathbb{E}[f(OPT_{i-1}) - f(OPT_{i})] &= 0 \nonumber\\
    &\leq \frac{\bar{b}_i}{2} \tag{by case 2 condition} \\
    &\leq \frac{\mathbb{E}[b_i]}{2} + rad \tag{using concentration}\\
    &
    =\frac{1}{2}\mathbb{E}[f(Y_i) - f(Y_{i-1})] + rad \tag{by \eqref{eq:eai:case2}}\\
    &\leq \frac{1}{2}\mathbb{E}[f(Y_i) - f(Y_{i-1})] + 5rad. \nonumber
\end{align}

Now consider that $u_i \in OPT$.  
By definition of $OPT_{i-1}$, $X_{i-1} \subseteq OPT_{i-1}$. Since, $u_i \notin X_{i-1}$. Then, $X_{i-1} \subseteq OPT_{i-1} \setminus \left\{u_i\right\}$. Thus by submodularity in expectation,

\begin{equation} \label{submodcase2part2}
    \mathbb{E}[ f\left(X_{i-1} \cup \left\{u_i\right\}\right) ]  - \mathbb{E}[ f\left(X_{i-1}\right)] \geq \mathbb{E}[f(OPT_{i-1} \setminus \left\{u_i\right\}) \cup \left\{u_i\right\}))] - \mathbb{E}[f(OPT_{i-1} \setminus \left\{u_i\right\})]. 
\end{equation}
This allows us to finish the bound with
\begin{align}
    \mathbb{E}[f(OPT_{i-1})] - \mathbb{E}[f(OPT_{i})] &= \mathbb{E}[f(OPT_{i-1})] - \mathbb{E}[f(OPT_{i-1} \setminus \left\{u_i\right\})] \tag{by \eqref{OPT_i}}\\
    &= \mathbb{E}[f(OPT_{i-1} \setminus \left\{u_i\right\}) \cup \left\{u_i\right\}))] - \mathbb{E}[f(OPT_{i-1} \setminus \left\{u_i\right\})] \nonumber\\
    &\leq \mathbb{E}[ f\left(X_{i-1} \cup \left\{u_i\right\}\right) ]  - \mathbb{E}[ f\left(X_{i-1}\right)] \tag{by \eqref{submodcase2part2}}\\
    &= \mathbb{E}[a_i] \tag{by def. of $a_i$}\\
    &\leq \bar{a}_i + 2rad \tag{using concentration}\\
    &\leq \frac{\bar{b}_i}{2} + 2rad \tag{condition for case 2}\\
    &\leq \frac{1}{2}\mathbb{E}[b_i] + 3rad \tag{using concentration}\\
    %
    &= \frac{1}{2}\mathbb{E}[f(Y_i) - f(Y_{i-1})] + 3rad \tag{by \eqref{eq:eai:case2}}\\
    &\leq \frac{1}{2}\mathbb{E}[f(Y_i) - f(Y_{i-1})] + 5rad. \nonumber
\end{align}

\end{proof}

\section{Benchmarks}
We now discuss benchmarks to assess the performance of our proposed algorithm.

\subsection{Optimal Bandit} 
The optimal bandit (OPT) requires the optimal set of arms as an input, and it only exploits this set throughout the time until $T$, see Algorithm\ref{alg:OPT}. The optimal set should be known in advance, or found using some offline algorithm. 

\begin{algorithm}
\caption{OPT}\label{alg:OPT}
\begin{algorithmic}
\Require horizon $T$, solution $S^{\star}$ 
\For{step time $t \in\{1, \ldots, T\}$}
    \State Play $S^{\star}$
\EndFor
\end{algorithmic}
\end{algorithm}

\subsection{Random Bandit}  
The random bandits (RND), plays at each time step a random subset of $\Omega$, where each arm is sampled independently with probability $\frac{1}{2}$, see Algorithm \ref{alg:RND}. The random algorithm in the offline setting has $\frac{1}{4}$-approximation guarantee \citep{feige2011maximizing}.

\begin{algorithm}
\caption{RND}\label{alg:RND}
\begin{algorithmic}
\Require Set of base arms $\Omega$, horizon $T$ \\
$n \leftarrow|\Omega|$ 
\For{step time $t \in\{1, \ldots, T\}$}
    \State $S^{(t)} \leftarrow \emptyset$
    \For{$i \in\{1, \ldots, n\}$}
        \State \textbf{with probability} $\frac{1}{2}$ \textbf{do} 
        \State \quad \quad $S^{(t)} \leftarrow S^{(t)} \cup\{u_i\}$ 
    \EndFor
    \State Play $S^{(t)}$
\EndFor
\end{algorithmic}
\end{algorithm}

\subsection{R-ETCG Bandit}  
Explore-then-commit greedy (ETCG) \citep{nie2022explore} is an online algorithm for monotone submodular maximization under full-bandit feedback, with proven guarantees in the monotone setting. The submodular monotone maximization, only makes sense when it is under constraint, otherwise the agent will pick all the arms as long as adding an arm is always beneficial. Therefore, the online algorithms for monotone submodular maximization require at least an extra input, such as the cardinality constraint $k$. Thus, to make applicable in our unconstrained non-monotone setting, 
we define random ETCG (R-ETCG), which initially generates a random cardinality budget $k \sim \mathcal{U}\left\{0,n\right\}$, then finds the best $k$ arms, see Algorithm  \ref{alg:cap}.

\begin{algorithm}
\caption{R-ETCG}\label{alg:cap}
\begin{algorithmic}
\Require Set of base arms $\Omega$, horizon $T$ \\
Initialize $S^{(0)} \leftarrow \emptyset, n \leftarrow|\Omega|, k \leftarrow \mathcal{U}\left\{0,n\right\}$ \\
Initialize $m \leftarrow \left \lceil \left(\frac{T \sqrt{2 \log (T)}}{n+2 n k \sqrt{2 \log (T)}}\right)^{2 / 3} \right \rceil$

\For{phase $i \in\{1, \ldots, k\}$}

    \For{arm $a \in \Omega \backslash S^{(i-1)}$}
        \State Play $S^{(i-1)} \cup\{a\}$ $m$ times
        \State Calculate the empirical mean $\bar{f}\left(S^{(i-1)} \cup\{a\}\right)$
    \EndFor
    
    \State $a_i \leftarrow \arg \max _{a \in \Omega \backslash S^{(i-1)}} \bar{f}\left(S^{(i-1)} \cup\{a\}\right)$
    
    \State $S^{(i)} \leftarrow S^{(i-1)} \cup\left\{a_i\right\}$

\EndFor
\For{remaining time} 
    \State Play $S^{(k)}$
\EndFor
\end{algorithmic}
\end{algorithm}

\section{More Experimental Evaluations}

In this section, we empirically evaluate our RGL algorithm in another non-monotone setting. We compare our method to the exact optimal solution and compute the empirical mean over different repetitions of the cumulative full regret instead of the cumulative $\frac{1}{2}$-regret, defined as follows, 
$$
\bar{\mathcal{R}}(T) = \frac{1}{rep} \sum_{n=1}^{rep} \sum_{t=1}^T ( f\left(OPT\right)- f\left(S_{t}\right)).
$$

We test for $n=8$ base arms, $T \in \left\{10^{2}, 10^{3},  10^{4} ,10^{5}, 10^{6}\right\}$ time horizon. We average our experiments over $rep = 9$ repetitions. We average the instantaneous rewards over a window of size 50.

\subsection{Linear Reward Minus Cost}
\label{linear_minus_cost}
We test the algorithms on a non-monotone stochastic function of the chosen set $X$, defined as follows,

\[
    f(X)= 
\begin{cases}
    1,  & \text{if } X =  \{5,6,7,8\} \\
    min(max(\sum_{a \in X} r(a) - \frac{|X|}{k^{\star}},0),1), & \text{  otherwise}
    
\end{cases}
\]

where $r(a)$ is the stochastic reward function of an individual arm where $\forall X, \forall a \in X, r(a) \in [0,1]$. In fact, we choose $r(a) \sim min(max(\mathcal{N}(\mu_a, \sigma),0),1)$, where $\forall X, \forall a \in X, \mu_a \in [0,1]$.

We fix an oracle constant of the submodular function $k^{\star} = 6$. We choose $\sigma = 0.02$, and $\mu$ the vector of all the $\mu_a$s, such as $\mu$ values are arranged from 0 to 0.35 with a step of 0.05. It can be easily verified that the set $\{5, 6, 7, 8\}$ is the optimal subset of arms. 

From Fig. \ref{figure3}, it can be seen that our proposed algorithm RGL is the only one that reaches the optimum among the above defined benchmarks (middle plot) and it has the least cumulative regret in terms of time horizon T (left plot). In terms of time step t (right plot), similarly to RND, RGL starts by a higher cumulative reward compared to R-ETCG, which is explainable by the relatively long early exploration phase of RGL, however, in later time steps, RGL outperforms R-ETCG, by reaching less cumulative regret, by exploiting the optimum set of arms.

\begin{figure*}[t]
\begin{center}
\includegraphics[width=14.5cm]{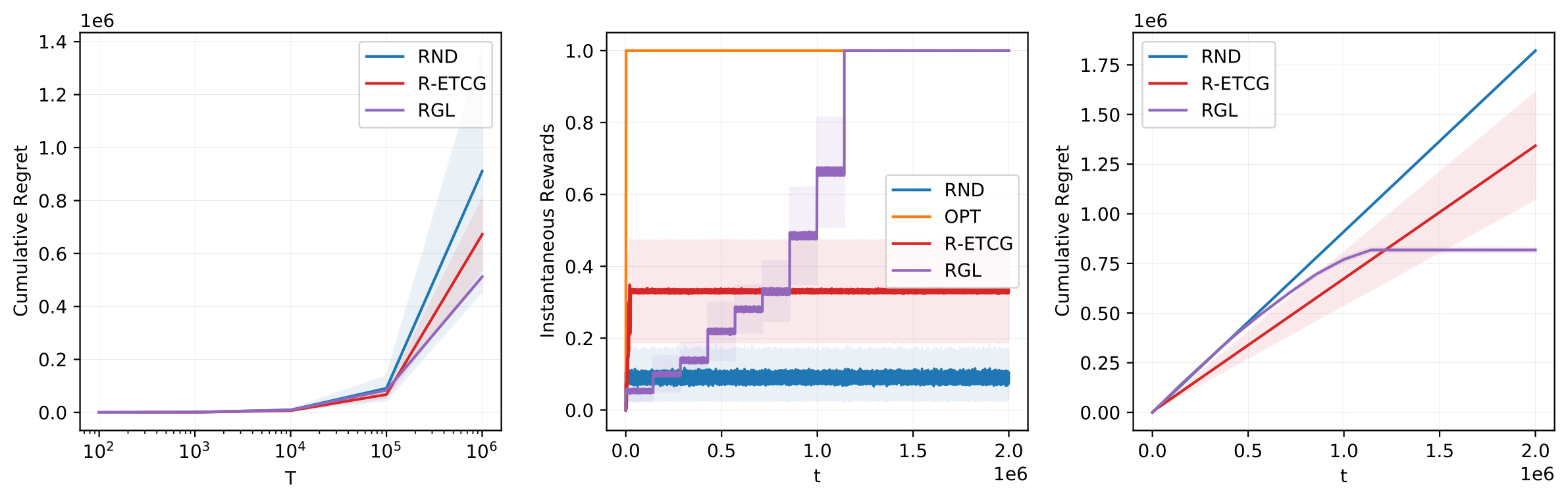}
\end{center}
\vspace{-.2in}
\caption{\small Comparison results for the non-monotone stochastic reward function. From left to right, the plots show cumulative regret as a function of time step $T$, instantaneous rewards as a function of time step $t$, and cumulative regret as a function of time horizon $t$, respectively.}
\label{figure3}
\end{figure*}

\subsection{Revenue Maximization over Social Networks}
\label{revenue_maximizatio_social}

In several real-world scenarios, non-monotone objectives are more meaningful.  For example, for revenue maximization over social networks, it is more meaningful to optimize the total revenue (influence minus costs; non-monotone) rather than the influence alone (monotone) with a budget as a constraint. Solutions to the latter will use all the budget, while the revenue-maximizing solution might use only a portion.

We test RGL on a non-monotone revenue maximization over social networks via influence maximization minus the costs. Influence maximization is indeed a submodular maximization problem which becomes non-monotone when we subtract the cost of adding nodes (Appendix \ref{inf-max}). 

We use the Karate network, which includes $34$ nodes, with an oracle function $f$, where for a subset of nodes $S$,
$$ 
f(S) =  \mathcal{N}(\sum_{c \in \mathcal{C}} \max_{a \in S \cap c} d(a), \sigma) - \alpha |S|, 
$$
where $\mathcal{C}$ refers to the set of communities, $d(a)$ is the degree of node a, $\mathcal{N}$ is the normal distribution, and $\alpha$ is a positive constant which depends on the cost. As shown in Fig. \ref{figure-revenue}, RGL outperforms all the other algorithms, as it has the lowest cumulative regret and almost reaches the optimal instantaneous rewards.

\begin{figure*}[t]
\begin{center}
\includegraphics[width=14.5cm]{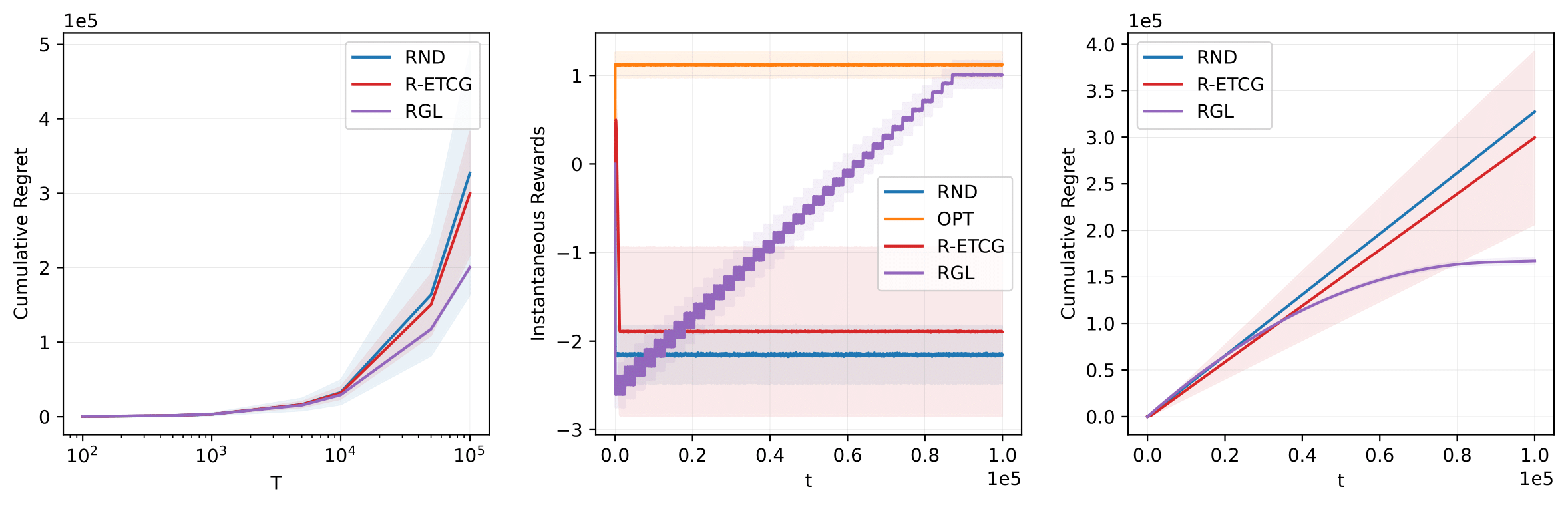}
\end{center}
\vspace{-.2in}
\caption{\small Revenue Maximization over Social Networks. From left to right, the plots show cumulative regret as a function of time step $T$, instantaneous rewards as a function of time step $t$, and cumulative regret as a function of time horizon $t$, respectively.}
\label{figure-revenue}
\end{figure*}

\end{document}